\documentclass[letterpaper, 10 pt,  conference]{ieeeconf}  

\IEEEoverridecommandlockouts                              
\usepackage{graphicx}
\graphicspath{{images_note/}}
\usepackage{subcaption}
\usepackage{amsmath}
\usepackage{verbatim}
\usepackage{amsfonts}
\usepackage{multirow}
\usepackage{changepage}
\usepackage{cite}
\usepackage{tabularx}
\usepackage{adjustbox}
\usepackage{mathtools}
\usepackage{amssymb}
\usepackage{soul}
\usepackage{url}
\usepackage{color}
\usepackage{algorithm} 
\usepackage{algpseudocode} 

\def\prob{\mathbb{P}}
\def\expt{\mathbb{E}}
\def\real{\mathbb{R}}

\def\natural{\mathbb{N}}
\DeclareMathOperator*{\argmax}{arg max}

\newcommand{\subscr}[2]{#1_{\textup{#2}}}
\newcommand{\supscr}[2]{#1^{\textup{#2}}}

\newcommand{\seqdef}[2]{\{#1\}_{#2}}

\newcommand\oprocendsymbol{\hbox{$\square$}}
\newcommand\oprocend{\relax\ifmmode\else\unskip\hfill\fi\oprocendsymbol}

\renewcommand{\baselinestretch}{0.97}

\newcommand\bit[1]{\textit{\textbf{#1}}}

\def \bf {\textbf{}}
\def \it {\textit{}}

\def \mc {\mathcal}

\newtheorem{theorem}{Theorem}

\newtheorem{lemma}{Lemma}

\newtheorem{remark}{Remark}
\newtheorem{example}{Example}

\newtheorem{definition}{Definition}

\renewcommand{\theenumi}{(\roman{enumi}}

\allowdisplaybreaks

\title{\LARGE \bf
Deterministic Sequencing of Exploration and Exploitation \\ for Reinforcement Learning
}

\author{Piyush Gupta and Vaibhav Srivastava
\thanks{This work has been supported in part by NSF Award IIS-1734272, CMMI-1940950, and ECCS-2024649.}
\thanks{{Piyush Gupta (guptapi1@msu.edu) and Vaibhav Srivastava (vaibhav@egr.msu.edu) are with the Department of Electrical and Computer Engineering, Michigan State University, East Lansing, Michigan, 48824, USA.}}
}

\begin{document}

\maketitle

\begin{abstract}
We propose Deterministic Sequencing of Exploration and Exploitation (DSEE) algorithm with interleaving exploration and exploitation epochs for model-based RL problems that aim to simultaneously learn the system model, i.e., a  Markov decision process (MDP), and the associated optimal policy. During exploration, DSEE explores the environment and updates the estimates for expected reward and transition probabilities. During exploitation, the latest estimates of the expected reward and transition probabilities
are used to obtain a robust policy with high probability. We design the lengths of the exploration and exploitation epochs such that the cumulative regret grows as a sub-linear function of time.
\end{abstract}

\section{Introduction}

Reinforcement Learning (RL) is used in solving complex sequential decision-making tasks in uncertain environments such as motion planning for robots~\cite{xin2017application, gupta2020towards}, personalized web services~\cite{barto2017some, ferretti2016automatic}, and the design of decision-support systems for human-supervisory control~\cite{gupta2022structural, zhou2006supervisory, gupta2019optimal}. Markov decision processes (MDPs)~\cite{sutton2018reinforcement} provide a natural framework for optimal decision-making under uncertainty and are used to model and solve numerous model-based RL problems. The objective of these problems is to simultaneously learn the system model and the optimal policy. While MDP formulation accounts for environment uncertainty by using stochastic models, MDP policies are known to be sensitive to errors in these stochastic models~\cite{ bertuccelli2008robust, 9654941}.

In many safety-critical systems, robust MDPs~\cite{wiesemann2013robust, nilim2004robust} are used to mitigate performance degradation due to uncertainty in the learned MDP. However, to reduce the system uncertainty, the agent must explore the environment and visit parts of the state space associated with high estimation uncertainty. Most often, RL algorithms use simple randomized methods to explore the environment, e.g. applying $\epsilon$-greedy policies~\cite{watkins1989learning} or adding random noise to continuous actions~\cite{lillicrap2015continuous}. 
The objective of the robust MDPs is conflicting with the exploration objective, { i.e., robust policy avoids the unexplored regions of the state space to optimize the worst-case performance while the objective of the exploration is to reduce the system uncertainty by visiting unexplored regions of the state space}. Therefore, to balance the trade-off between learning the MDP and designing a robust policy, we design a Deterministic Sequencing of Exploration and Exploitation (DSEE) algorithm, in which exploration and exploitation epochs of increasing lengths are interleaved.

There exist efficient algorithms for solving RL problems with provable bounds on the {sample complexity~\cite[Definition 1]{strehl2008analysis}}. In~\cite{strehl2008analysis}, authors analyze the Model-based Interval Estimation (MBIE) algorithm that applies confidence
bounds to compute an optimistic policy and show that the algorithm is {PAC-optimal~\cite[Definition 2]{strehl2008analysis}. They provide an upper bound on the algorithm's sample complexity given by $O\left( \frac{|\mathcal{S}|^2|\mathcal{A}|}{(1-\gamma)^6\epsilon^3}\log(\delta^{-1})\right)$ which is the maximum number of time steps until when the MBIE policy is not $\epsilon-$optimal with at least probability $1-\delta$, where $|\mathcal{S}|$, $|\mathcal{A}|$, are the cardinality of the state space and action space, respectively, $\gamma$ is the discount factor, and $\epsilon, \delta \in (0,1)$ are pre-defined constants.} A similar bound on the sample complexity is obtained for the R-max algorithm~\cite{kakade2003sample} which distinguishes the ``known" and ``unknown" states based on how often they have been visited. It explores by acting to maximize rewards under the assumption that unknown states deliver the maximum reward. UCRL2 algorithm~\cite{JMLR:v11:jaksch10a} relies on optimistic bounds on the reward functions and probability density functions and enjoys near-optimal regret bounds. A review of model-based RL algorithms with provable finite time guarantees can be found in \cite[Chapter 38]{lattimore2020bandit}. A major drawback of these algorithms is that they consider optimism in the face of uncertainty and hence, are not robust to the estimation uncertainties. Furthermore, these algorithms with random exploration might lead to a bad user experience in applications in which the RL agent seeks to learn human preferences for system optimization.

 To address these shortcomings, we propose a DSEE algorithm for model-based RL in which we design a deterministic sequence of exploration and exploitation epochs. The DSEE approach has been used in multi-arm bandit problems \cite{SV-KL-QZ:13, HL-KL-QZ:13,LW-VS:17i, nayyar2016regret} and multi-robot coordination~\cite{wei2021multi}. It allows for differentiation between exploration and exploitation epochs. The announced exploration may lead to a better user experience for the agents {(especially for human agents)} than random exploration at any time. For example, many personalized web services calibrate their recommendations intermittently by announced exploration, i.e., through surveys and user selection. Another advantage of the DSEE algorithm is that it allows for efficient exploration of the environment in multi-agent systems. Specifically, in multi-agent systems, exploration can be well-planned to cover all regions of the state-space through agent coordination which can be easily arranged  due to the deterministic structure of exploration and exploitation.

We design the DSEE algorithm with alternating sequences of exploration and exploitation. In exploration epochs, the algorithm learns the MDP, while in exploitation epochs, it uses a robust policy based on the learned MDP and the associated uncertainty. We design the lengths of the exploration and exploitation epochs such that the cumulative regret grows as a sub-linear function of time. 

The major contributions of this work are twofold: (i) we propose a DSEE algorithm for model-based RL problems and (ii) {we design the lengths of the exploration and exploitation epochs} such that 
the cumulative regret for the DSEE algorithm grows as a sub-linear function of time.
 
This manuscript is structured as follows: in Section~\ref{Problem Setup}, we provide background and formulate the problem. 
In Section~\ref{section:DSEE}, we provide an overview of the DSEE algorithm.
In Section~\ref{Section:Analysis}, we
analyze the DSEE algorithm and design the exploration and exploitation epochs such that the cumulative regret grows sub-linearly with time. We conclude in Section~\ref{Conclusions}.

\section{Background and Problem Formulation}\label{Problem Setup}

We focus on the model-based RL problems which aim to simultaneously learn the system model, i.e., a Markov decision process (MDP), and the associated optimal policy. We seek to design policies that are robust to uncertainty in the learned MDP. However, learning the MDP requires visiting parts of the state space associated with high uncertainty in estimates and has exactly the opposite effect of a robust policy. Therefore, to balance the trade-off between learning the MDP and designing a robust policy, we design a DSEE algorithm, in which exploration and exploitation epochs of increasing lengths are interleaved. In exploration epochs, the algorithm learns the MDP, while in exploitation epochs, it uses a robust policy based on the learned MDP and the associated uncertainty.


 Consider an MDP $(\mathcal{S}; \mathcal{A}, R, \mathbb{P}, \gamma)$, where $\mathcal{S}$ is the state space, $\mathcal{A}$ is the action space, {reward $R(s,a)$, for each $(s,a) \in \mathcal{S} \times \mathcal{A}$, is a random variable with support $[0, R_{\max}]$}, $\mathbb{P} : \mathcal{S \times A} \rightarrow \Delta_{|\mathcal{S}|}$ is the transition distribution, and $\gamma \in (0, 1)$ is the discount factor. Here, $\Delta_{|\mathcal{S}|}$ represents probability simplex in $\real^{|\mathcal{S}|}$, $|\cdot|$ represents the cardinality of a set. {Let $\overline{R}(s,a)$ be the expected value of $R(s,a)$}. We consider a finite MDP setting in which $|\mathcal{S}|$ and $|\mathcal{A}|$ are finite.

We assume that the rewards $R$ and the state transition distribution $\mathbb{P}$ are unknown a priori. Hence, during exploration, we estimate  $\overline R$ and $\mathbb{P}$ using online observations. {Let $(s, a)$ be any state-action pair where $s \in \mathcal{S}$ and $a \in \mathcal{A}$}. At any time $t$, let $n_t(s,a)$ be the number of times state-action pair $(s,a)$ is observed until time $t$. For each $(s,a)$, the empirical mean estimates $\hat{R}_t(s,a)$ and $\hat{\mathbb{P}}_t(s'|s,a)$, $s' \in \mathcal{S}$ are:
\begin{align}\label{eq:empirical reward}
    \hat{R}_t(s,a) &= \frac{1}{n_t(s,a)}\sum_{i=1}^{n_t(s,a)}r_i(s,a),  \text{and } \\
\label{eq:empirical prob} \hat{\mathbb{P}}_t(s'| s,a) &= \frac{n_t(s,a,s')}{n_t(s,a)},
\end{align}
respectively, where $r_i(s,a)$ is the immediate reward obtained in $(s,a)$ during observation $i \in \{1, \ldots, n_t(s,a)\}$ until time $t$ and $n_t(s,a,s')$ is the number of times the next state $s'$ is observed from $(s,a)$ out of $n_t(s,a)$ times.

Oftentimes, the uncertainty in probability transition matrices and mean reward function can be large, especially in the initial stages of learning due to limited observation data, which may lead to sub-optimal policies. Robust MDPs~\cite{wiesemann2013robust} mitigate the sub-optimal performance arising from this uncertainty by optimizing the worst-case performance over given uncertainty sets for reward function and probability transition matrices to obtain a robust policy.
%
%
%
Given, at time $t$, uncertainty sets $\mathcal{R}_t^U$ and ${\mathcal{P}_t^U}$ containing $\overline R$ and $\prob$, respectively,  
the robust MDP solves the following robust Bellman equation: 
    \begin{align}\label{Robust_Bellman}
 V^{R}_t(s) &=\max_{a\in \mathcal{A}} \min_{\tilde{R}_t \in \mathcal{R}^U_t, \ {\tilde{\mathbb{P}}}_t \in \mathcal{P}_t^U}\Big\{  \tilde{R}_t(s,a) + \nonumber\\
 &\quad \quad \quad \quad \quad \gamma\sum_{s'}{\tilde{\mathbb{P}}}_t(s'|s,a)V^{R}_t(s')  \Big\},
    \end{align}
to obtain a robust policy $\hat{\pi}^R_t = \argmax _{a \in \mathcal{A}}V^R_t$, which optimizes the worst-case performance through minimization with respect to the uncertainty sets $\mathcal{R}_t^U$ and  ${\mathcal{P}_t^U}${, where $V^R_t$ is the robust value function}.

The choice of these uncertainty sets 
are critical for the performance of the robust algorithm. A poor modeling choice can increase the computational complexity and result in a highly conservative policy~\cite{bertuccelli2012robust, 9654941}. 
To avoid these issues, during the exploitation epoch of the DSEE,  we construct these uncertainty sets based on the  estimates $\hat{R}_t$ and $\hat{\mathbb{P}}_t$ from the previous exploration epochs and Hoeffding bounds~\cite{wainwright2019high} for $\hat R_t$ (Lemma~\ref{lemma:reward_bound}) and $\hat{\mathbb{P}}_t$ (Lemma~\ref{lemma:probability_bound}). Subsequently, we utilize robust MDP to learn a policy that is robust to the estimation uncertainties with high probability. 
The convergence of the robust MDP with uncertain transition matrices to the uncertainty-free MDP can be shown under the assumption that the uncertainty sets converge to singleton estimates almost surely~\cite{iyengar2005robust, 9654941}.

\begin{definition}[\bit{Instantaneous and Cumulative Regret}] \label{definition1}
For a discounted and ergodic RL{~\cite{puterman2014markov}}, consider an algorithm {$\mathbb{A}$} that, at the end of the $(t -1)$-th step, returns a policy $\pi_t$ to be applied in the $t$-th step. For any state $s \in \mathcal{S}$, let $V^{*}(s)$ and $V^{\pi_t}(s)$ be the optimal value of the state and its value under the policy $\pi_t$, respectively. At any time $t$, the instantaneous regret $\mathfrak{R}(t)$ of the algorithm $\mathbb{A}$ is given by:
\begin{equation}~\label{eq: instantaneous_regret_definition}
    \mathfrak{R}(t)=\Vert V^*(s) - V^{\pi_t}(s) \Vert_{\infty},
\end{equation}
where $\Vert \cdot \Vert_{\infty}$ denotes the $L^{\infty}$-norm of a vector, and the cumulative regret $\mathbf{R}_T$  until time horizon $T$ is given by:
\begin{equation}~\label{eq: regret_definition}
    \mathbf{R}_T=\sum_{t=1}^{T} \mathfrak{R}(t) =\sum_{t=1}^{T}\Vert V^*(s) - V^{\pi_t}(s) \Vert_{\infty}. 
\end{equation}
\end{definition}
We design the exploration and exploitation epochs of the DSEE algorithm such that its cumulative regret grows as a sub-linear function of time. In the next section, we provide an overview of the DSEE algorithm.

\section{DSEE Algorithm}\label{section:DSEE}

\begin{algorithm}
	\caption{Deterministic Sequencing of Exploration and Exploitation (DSEE)}
	\label{algo1}
	\textbf{Input:} Set of states $\mathcal{S}$, Set of actions $\mathcal{A}$, Initial State $s_0$;\\
 \textbf{Set:} $\eta>1$, Sequences $\seqdef{\epsilon_j}{j \in \natural}$, $\seqdef{\delta_j}{j \in \natural}$, $\supscr{s}{end}_0 = s_0$, $s=s_0$;
 
  \textbf{Set:} $t =0$, $n(s,a)=0, n(s,a,s')=0, \mathrm{S}(s,a) =0, \forall s,a,s'$;
	\begin{algorithmic}[1]
		\For {epoch $j=1,2,\ldots$}
		
\textit{\% Exploration phase:}

\State $\rho_j \!\leftarrow\!\frac{\epsilon_j}{ 4+ \frac{2R_{\max}\gamma}{(1-\gamma)^2}}$;
\State $\mu \leftarrow 2\left[\log(2^{|\mathcal{S}|}-2) + \log\left(\frac{2|\mathcal{S}||\mathcal{A}|}{\delta_j}\right)\right]$;
\State $U_j  \leftarrow \max \Big\{ \frac{{(R_{\max})^2} \log(\frac{4|\mathcal{S}||\mathcal{A}|}{\delta_j})}{2\rho_j^2}   ,  \frac{\mu}{\rho_j^2}  \Big\}$

 \While{{$n(s,a)  < U_j, \; \forall (s,a)$ }
 \State $t \leftarrow t+1$;
		\State {Pick $a \sim \text{UNIF}(\mathcal{A})$ in current state $s$}} 
		\State Observe reward $\mathrm R$ and the next state $s'$
		
		\If{$\supscr{s}{end}_{j-1}$ has been visited in epoch $j$
		} 
		
		\State $n(s,a) \leftarrow n(s,a) +1$;
        \State $n(s,a, s') \leftarrow n(s,a, s') +1$;
        \State $\mathrm S(s,a) = \mathrm S(s,a) + \mathrm{R}$; 
        
		\EndIf
		\State $s \leftarrow s'$;
	\EndWhile
	
	\State $\supscr{s}{end}_{j} \leftarrow s$; 
	
	\State $\hat{R}_t(s,a) = \frac{\mathrm{S}(s,a)}{n(s,a)}, \ \forall (s,a)$;
		\State $\hat{\mathbb{P}}_t(s'| s,a) = \frac{n(s,a,s')}{n(s,a)}, \ \forall (s,a)$;
		
\textit{\% Exploitation phase:}
		
				\State Construct uncertainty sets  $\mathcal{R}_t^U$ and ${\mathcal{P}_t^U}$ using~\eqref{eq:uncertainty-sets}; 
		
		\State Compute ${V}^R_{t}(s)$ and $\hat{\pi}_t^R $ using ~\eqref{Robust_Bellman}; 
		
		\State Implement $\hat{\pi}_t^R $ for $\lceil \eta^j \rceil$ time steps; 
		
		\State $t \leftarrow t + \lceil \eta^j \rceil$;


		\EndFor
	\end{algorithmic} 
\end{algorithm} 

We design the DSEE algorithm for model-based RL under the following assumptions:
\begin{enumerate}
\item[(A1)] State space $\mathcal{S}$ and action space $\mathcal{A}$ are finite sets. 
\item[(A2)] The MDP is ergodic {under the uniform policy $\pi$,  i.e., 
under a policy $\pi$ that, in every state $s$, randomly selects the actions from $\mathcal{A}$ with equal probability}, the MDP admits a unique stationary distribution {$\phi_{\pi}(s): \mathcal{S} \rightarrow \Delta_{|\mathcal{S}|}$}, with $\phi_{\pi}(s) > 0$ for all $s$. 
\end{enumerate}

Ergodic MDP~\cite{puterman2014markov} (assumption (A2)) is a common assumption.
It ensures that the stationary distribution is independent of the initial distribution and all states are recurrent, i.e., each state $s$ is visited infinitely often and $\phi_{\pi}(s)>0$. We use this assumption to estimate the number of times each state is visited in $N$ time steps.  


Algorithm~\ref{algo1} shows an overview of the DSEE algorithm. In the DSEE algorithm, we design a sequence of alternating exploration and exploitation epochs.  Let $\alpha_i$ and $\beta_i$ be the lengths of the $i$-th exploration and exploitation epoch, respectively, where $i \in \mathbb{N}$. During an exploration epoch, we uniformly sample the action in {the current} state and update the estimates $\hat{R}_t$ and $\hat{\mathbb{P}}_t$. For a given sequence of $\seqdef{\epsilon_i}{i\in \natural}$ and $\seqdef{\delta_i}{i \in \natural}$ that we design in Section~\ref{Section:Analysis}, the length of the exploration epoch {$\alpha_i$} is determined to reduce the estimation uncertainty such that $\mathbf{P}(\Vert V^*(s) - {V^{\hat{\pi}^R_t}(s)}  \Vert_{\infty} \le \epsilon_i) \ge 1-\delta_i$ after the epoch, {where $\mathbf{P}(\cdot)$ 
denotes the probability measure,
and $V^{\hat{\pi}^R_t}(s)$ is the value of state $s$ under the robust policy $\hat{\pi}_t^R$ }.
In DSEE, we choose exponentially increasing lengths of the exploitation epochs {$\beta_i$}. During the exploitation epoch, we utilize the  estimates $\hat{R}_t$ and $\hat{\mathbb{P}}_t$ from previous exploration epochs and construct the uncertainty sets $\mathcal{R}^U$ and  $\mathcal{P}^U$ at time $t$. We use these uncertainty sets with a robust Bellman equation to learn a policy that is robust to the estimation uncertainties with high probability. In next section, we analyze the DSEE algorithm, and design the sequence of $\seqdef{\epsilon_i}{i\in \natural}$ and $\seqdef{\delta_i}{i \in \natural}$, such that the cumulative regret~\eqref{eq: regret_definition} grows as a sub-linear function of time.


\section{Analysis of DSEE algorithm}\label{Section:Analysis}

We now characterize the regret of the DSEE algorithm under the assumptions {(A1-A2)} and design the exploration and exploitation epochs.
The optimal value {$V^*(s_t)$} of the state $s_t$ is given by:
\begin{equation}
    V^*(s_t) = {\overline{R}}(s_t, \pi^*(s_t)) + \gamma\expt{\left[V^*(s_{t+1})|s_t, \pi^*(s_t)\right]},
\end{equation}
where $\pi^*$ is an optimal policy that satisfies:
\begin{equation}
    \pi^*(s_t)=\argmax_{a_t} \left\{   {\overline{R}}(s_t, a_t) + \gamma\expt{\left[V^*(s_{t+1})|s_t, a_t\right]} \right\}.
\end{equation}
We define an approximate optimal value function $\hat{V}_t$ that utilizes the estimates $\hat{R}_t$ and $\hat{\mathbb{P}}_t$ at time $t$. Therefore, $\hat{V}_t(s_t)$ is given by:
\begin{equation}\label{eq:vhat}
    \hat{V}_t(s_t) = \hat{R}_t(s_t, \hat{\pi}_t(s_t)) + \gamma\hat{\mathbb{E}}{\left[\hat{V}_{t}(s_{t+1})|s_t, \hat{\pi}_t(s_t)\right]},
\end{equation}
where $\hat{\mathbb{E}}{\left[\hat{V}_{t}(s_{t+1})|s_t, \hat{\pi}_t(s_t)\right]}$ is used to denote $\sum_{s_{t+1}}\hat{\mathbb{P}}_t(s_{t+1}|s_t,\hat{\pi}_t(s_t))\hat{V}_{t}(s_{t+1})$ and $\hat{\pi}_t$ is an optimal policy for the approximate optimal value function given by:
\begin{equation}
    \hat{\pi}_t(s_t)=\argmax_{a_t} \left\{  \hat{R}_t(s_t, a_t) + \gamma\hat{\mathbb{E}}{\left[\hat{V}_{t}(s_{t+1})|s_t, a_t\right]} \right\}.
\end{equation}

\begin{theorem}[\bit{Concentration of robust value function}]\label{thm1}
{Let $\Vert \cdot \Vert_1$ denote the $L^{1}$-norm of a vector.} For any given $\epsilon_t, \delta_t \in (0,1)$, there exists an $n \in  O\left(\frac{|\mathcal{S}|}{\epsilon_t^2} + \frac{1}{\epsilon_t^2}\log\left(\frac{|\mathcal{S}||\mathcal{A}|}{\delta_t}\right)\right)$ such that if each state-action pair $(s,a)$ is observed $n_t(s,a) \ge n$ times until time $t$, then {for each state $s$,} the following inequality holds:
\begin{equation}
     \mathbf{P}\left(  \Vert V^*(s) - V^{\hat{\pi}_t^R}(s) \Vert_{\infty}\le \epsilon_t     \right) \ge 1-\delta_t,
\end{equation}
where $V^{\hat{\pi}^R_t}(s)$ is the value of state $s$ under the robust policy $\hat{\pi}_t^R = \argmax _{a \in \mathcal{A}}V^R_t$. The robust value function $V^R_t$ is defined in~\eqref{Robust_Bellman} with $\rho_t =\frac{\epsilon_t}{2}{\left( 2+ \frac{R_{\max}\gamma}{(1-\gamma)^2}  \right)^{-1}}$ and
\begin{equation}\label{eq:uncertainty-sets}
\begin{split}
    \mathcal{R}_t^U &=\left\{ {R^U(s,a)}: |{R^U(s,a)} - \hat{R}_t(s,a)| \le \rho_t,  \; \forall (s,a)  \right\} , \\
\mathcal{P}_t^U &=\left\{ {\mathbb{P}^U(s,a)} : \left\Vert {\mathbb{P}^U(s,a)} - \hat{\mathbb{P}}_t(s,a)\right\Vert_1\le \rho_t, \; \forall  (s,a)\right\}.
\end{split}
\end{equation}

\end{theorem}

\smallskip 
We prove Theorem~\ref{thm1} using the following Lemmas~\ref{lemma:reward_bound}-\ref{lemma:value_bound}.

\begin{lemma}[\bit{Concentration of rewards}]\label{lemma:reward_bound}
Suppose until time step $t$, the state-action pair $(s,a)$ is observed $n_t(s,a)$ times and bounded immediate rewards {$r_i(s,a)$}, $i \in \{1, \ldots, n_t(s,a)\}$, are obtained at these instances. Then the following inequality holds: 
\begin{equation}
    \mathbf{P}\left(\left| {\overline{R}}(s,a) - \hat{R}_t(s,a)\right|\le \epsilon_{t}^R\right) \ge 1 - \delta_{t}^R,
\end{equation}
where $\hat{R}_t(s,a)$ is the empirical mean reward defined in~\eqref{eq:empirical reward} and $\epsilon_{t}^R = \sqrt{\frac{{(R_{\max})^2} \log(2/\delta^R_{t})}{2n_t(s,a)}}$.
\end{lemma}

\begin{proof}
For brevity of notation, let $\epsilon_R$ and $\delta_R$ denote $\epsilon_{t}^R$ and $\delta^R_{t}$, respectively. For bounded random variables $r_i(s, a)$, using the Hoeffding bounds~\cite{wainwright2019high}, we have
\begin{equation}
    \mathbf{P}\left(\left|  {\overline{R}}(s,a) - \hat{R}_t(s,a)\right|\le \epsilon_R\right) \ge 1 - 2e^{-\frac{2n_t(s,a)\epsilon_R^2}{{(R_{\max})^2}}}.
\end{equation}
Choosing $\delta_R = 2e^{-\frac{2n_t(s,a)\epsilon_R^2}{{(R_{\max})^2}}}$, we get the desired result.
\end{proof}

\begin{lemma}[\bit{Concentration of transition probabilities}]\label{lemma:probability_bound}
Suppose until time step $t$, the state-action pair $(s,a)$ is observed $n_t(s,a)$ times and let $\mathbb{P}(s,a) \in \Delta_{|\mathcal{S}|}$ be the true transition probability distribution for $(s,a)$. 
Then {for any $(s,a)$}, the following inequality holds:
\begin{equation}
    \mathbf{P}\left(\left\Vert\mathbb{P}(s,a) - \hat{\mathbb{P}}_t(s,a)\right\Vert_1\le \epsilon_{t}^P\right) \ge 1 - \delta^P_{t},
\end{equation}
where $\Vert \cdot\Vert_1$ is the {$L^1$ norm of a vector} and $\hat{\mathbb{P}}_t(s,a)$ is the empirical transition probability vector with components $\hat{\mathbb{P}}_t(s'| s,a)$ defined in~\eqref{eq:empirical prob},
and $\epsilon_{t}^P = \sqrt{\frac{2[\log(2^{|\mathcal{S}|}-2) - \log(\delta^P_{t})]}{n_t(s,a)}}$. 
\end{lemma}

\begin{proof}
For brevity of notation, let $\epsilon_P$ and $\delta_P$ denote $\epsilon_{t}^P$ and $\delta^P_{t}$, respectively. Using~\cite[Theorem 2.1]{weissman2003inequalities}, we have:
\begin{equation}
    \mathbf{P}\left(\left\Vert\mathbb{P}(s,a) - \hat{\mathbb{P}}_t(s,a)\right\Vert_1\le \epsilon_P\right) \ge 1 - (2^{|S|}-2)e^{-\frac{n_t(s,a)\epsilon_P^2}{2}}.
\end{equation}
Setting $\delta_P = (2^{|S|}-2)e^{-\frac{n_t(s,a)\epsilon_P^2}{2}}$, yields the desired result. 
\end{proof}

Lemmas~\ref{lemma:reward_bound} and~\ref{lemma:probability_bound} provide concentration bounds on the reward and transition probability based on how often a state-action pair is visited.

\begin{lemma}\label{lemma:sampling_order}\bit{(Concentration of reward and transition probability functions)}
Let $\delta^R_{t}=\delta^P_{t} = \frac{\delta_t}{2|\mathcal{S}||\mathcal{A}|}$. Then, for any $\rho_t > 0$, there exists an $n \in O\left(\frac{|\mathcal{S}|}{\rho_t^2} + \frac{1}{\rho_t^2}\log\left(\frac{|\mathcal{S}||\mathcal{A}|}{\delta_t}\right)\right)$ such that when each state-action pair $(s,a)$ is observed $n_t(s,a) \ge n$ times, then the following statements hold for any $(s, a)$: 
\begin{enumerate}
    \item $\mathbf{P}\left(| {\overline{R}}(s,a) - \hat{R}_t(s,a)| \le \rho_t\right) \ge 1-\frac{\delta_t}{2|\mathcal{S}||\mathcal{A}|},$ \\
    \item $\mathbf{P}\left(\left\Vert\mathbb{P}(s,a) - \hat{\mathbb{P}}_t(s,a)\right\Vert_1\le \rho_t\right) \ge 1 -\frac{\delta_t}{2|\mathcal{S}||\mathcal{A}|}.$
\end{enumerate}
\end{lemma}

\begin{proof}
Using Lemmas~\ref{lemma:reward_bound} and~\ref{lemma:probability_bound}, we know that $| {\overline{R}}(s,a) - \hat{R}_t(s,a)| \le \rho_t$ and  $\left\Vert\mathbb{P}(s,a) - \hat{\mathbb{P}}_t(s,a)\right\Vert_1\le \rho_t$  holds for any $(s, a)$ with at least probability $1-\frac{\delta_t}{2|\mathcal{S}||\mathcal{A}|}$ for $\rho_t \ge \sqrt{\frac{({R_{\max})^2}\log(\frac{4|\mathcal{S}||\mathcal{A}|}{\delta_t})}{2n_t(s,a)}}$ and $\rho_t\ge \sqrt{\frac{2\left[\log(2^{|\mathcal{S}|}-2) - \log\left(\frac{\delta_t}{2|\mathcal{S}||\mathcal{A}|}\right)\right]}{n_t(s,a)}}$, respectively. Hence, we have that
\begin{equation}\label{eq:n(s)}
    n_t(s,a) \ge  
    \max\left\{ \frac{{(R_{\max})^2} \log(\frac{4|\mathcal{S}||\mathcal{A}|}{\delta_t})}{2\rho_t^2}   ,  \frac{\mu}{\rho_t^2}  \right\},
\end{equation}
where $\mu= 2\left[\log(2^{|\mathcal{S}|}-2) + \log\left(\frac{2|\mathcal{S}||\mathcal{A}|}{\delta_t}\right)\right]$, is sufficient to guarantee that $| {\overline{R}}(s,a) - \hat{R}_t(s,a)| \le \rho_t$ and  $\left\Vert\mathbb{P}(s,a) - \hat{\mathbb{P}}_t(s,a)\right\Vert_1\le \rho_t$  holds for any $(s, a)$ with at least probability $1-\frac{\delta_t}{2|\mathcal{S}||\mathcal{A}|}$. Hence, we can choose $n \in O\left(\frac{|\mathcal{S}|}{\rho_t^2} + \frac{1}{\rho_t^2}\log\left(\frac{|\mathcal{S}||\mathcal{A}|}{\delta_t}\right)\right)$ such that~\eqref{eq:n(s)} holds, and hence, the lemma follows.
\end{proof}

\begin{remark}[\bit{Concentration inequalities}]
The concentration inequalities in Lemmas~\ref{lemma:reward_bound}, \ref{lemma:probability_bound}, and~\ref{lemma:sampling_order} use a deterministic value of $n_t(s,a)$. However, these bound also apply if $n_t(s,a)$ is a realization of a random process that is independent of $\hat{R}_t(s,a)$ and $\hat{\mathbb{P}}_t(s,a)$, which would be the case in this paper. 
\end{remark}

\begin{remark}[\bit{Uncertainty set}]\label{remark:uncertainty set}
Using union bounds over all $(s, a)$ and Lemma~\ref{lemma:sampling_order}, 
it follows that uncertainty sets $\mathcal{R}^U_t$ and $\mathcal{P}^U_t$ stated in Theorem~\ref{thm1} are $\rho_t$-level uncertainty sets  for $R(s,a)$ and $\mathbb{P}(s,a)$, respectively, for each $(s,a) \in \mc S \times \mc A$ with at least probability $1-\delta_t$. Thus, the policy obtained at time $t$ in an exploitation epoch is robust to estimation uncertainties with at least probability $1-\delta_t$.
\end{remark}

\begin{lemma}[\bit{Loss in robust value function}]\label{lemma:value_bound}
Suppose $| {\overline{R}}(s,a) - \hat{R}_t(s,a)| \le \rho_t$ and  $\left\Vert\mathbb{P}(s,a) - \hat{\mathbb{P}}_t(s,a)\right\Vert_1\le \rho_t$, for any $(s, a)$, at time $t$. Then $\mathcal{L}_t(s) = V^*(s) - {V^{\hat{\pi}_t^R}}(s)$ satisfies:
\begin{equation}
        \mathcal{L}_t(s) \le 2\rho_t\left(2 + \frac{R_{\max}\gamma}{(1-\gamma)^2}\right),
\end{equation}
where ${V^{\hat{\pi}_t^R}}(s)$ is the value of state $s$ under the robust policy $\hat{\pi}_t^R$ at time $t$.
\end{lemma}

\begin{proof} $\mathcal{L}_t(s)$ can be written as:
\begin{equation}
 \mathcal{L}_t(s) =   V^*(s) - V^R_t(s)     +   V^R_t(s)  - {V^{\hat{\pi}_t^R}}(s), 
\end{equation}
where $V^R_t(s)$ is the robust value of state $s$ and $\hat{\pi}_t^R = \argmax _{a \in \mathcal{A}}V^R_t$ at time $t$. Let $\tilde{R}^R_t \in \mathcal{R}_t^U$ and ${\tilde{\mathbb{P}}}^R_t \in \mathcal{P}_t^U$ be the worst-case reward and transition probability corresponding to minimization in~\eqref{Robust_Bellman} at time $t$. 
Since $\tilde{R}^R_t \in \mathcal{R}_t^U$ and $| {\overline{R}}(s,a) - \hat{R}_t(s,a)| \le \rho_t$, we have that 
 \begin{equation}
      | {\overline{R}}(s,a) - \tilde{R}_t^R(s,a)| \le 2\rho_t.
 \end{equation}
  Similarly, since  $\tilde{\mathbb{P}}^R_t \in \mathcal{P}_t^U$ and $\left\Vert\mathbb{P}(s,a) - \hat{\mathbb{P}}_t(s,a)\right\Vert_1\le \rho_t$,  we get:
  \begin{equation}
    \left\Vert\mathbb{P}(s,a) - \tilde{\mathbb{P}}^R_t(s,a)\right\Vert_1\le 2\rho_t.  
  \end{equation}
  
 Furthermore, by setting $\hat{R}_t = \tilde{R}_t^R$ and $\hat{\mathbb{P}}_t = \tilde{\mathbb{P}}_t^R$ in~\eqref{eq:vhat}, we can obtain an approximate optimal value function $\tilde{V}_t$ such that $\tilde{V}_t(s) = V_t^R(s)$ and the corresponding optimal policy $\tilde{\pi}_t(s) = \hat{\pi}_t^R(s)$, for all $s \in \mathcal{S}$. Therefore $\mathcal{L}_t(s) =  V^*(s) - \tilde{V}_t(s) + \tilde{V}_t(s) - {V^{\hat{\pi}_t^R}}(s)$, where $\tilde{V}_t(s)$ utilizes $\hat{R}_t(s,a)$ and $\hat{\mathbb{P}}_t(s,a)$ such that $| {\overline{R}}(s,a) -\hat{R}_t(s,a)| \le 2\rho_t$ and $ \left\Vert\mathbb{P}(s,a) - \hat{\mathbb{P}}_t(s,a)\right\Vert_1\le 2\rho_t$ holds, respectively.
 
 \cite{mastin2012loss} provides an upper bound on the loss $V^*(s) - V^{\hat{\pi}_t}(s)$, where the policy $\hat{\pi}_t$ is the optimal policy for an approximate value function where the  uncertainty is only in $\mathbb{P}$ and the estimate $\hat{\mathbb{P}}_t$ satisfies $\left\Vert\mathbb{P}(s,a) - \hat{\mathbb{P}}_t(s,a)\right\Vert_1\le 2\rho_t$. Therefore, it assumes that the true {mean} reward function is known ($\hat{R}_t(s,a) =  {\overline{R}}(s,a)$ in~\eqref{eq:vhat}). Again, setting $\hat{R}_t = \tilde{R}_t^R$ and $\hat{\mathbb{P}}_t = \tilde{\mathbb{P}}_t^R$, and by following the similar line of  analysis as in \cite{mastin2012loss} and considering bounded uncertainty in $ {\overline{R}}$, where the estimate $\hat{R}_t$ satisfies $|{\overline{R}}(s,a) -\hat{R}_t(s,a)| \le 2\rho_t$, we obtain 
 \begin{equation}
     \mathcal{L}_t(s) =  V^*(s) - \tilde{V}_t(s) + \tilde{V}_t(s) - {V^{\hat{\pi}_t^R}}(s) \le 4\rho_t + \frac{2R_{\max}\gamma \rho_t}{(1-\gamma)^2},
 \end{equation}
 where the loss contribution due to reward uncertainty in $V^*(s) - \tilde{V}_t(s)$ and $\tilde{V}_t(s) - {V^{\hat{\pi}_t^R}}(s)$ can be bounded by $2\rho_t$ each, and $\frac{2R_{\max}\gamma \rho_t}{(1-\gamma)^2}$ is an upper bound on the loss due to uncertainty in $\mathbb{P}$ as obtained in~\cite[Theorem 2]{mastin2012loss}. 
  \end{proof}

 Lemma~\ref{lemma:value_bound} provides the bounds on the loss in robust value function w.r.t. the optimal value function using the concentration bounds on the rewards and transition probabilities.

\bit{Proof of Theorem 1:} 
 Using Lemma~\ref{lemma:sampling_order}, we know that when each state-action pair $(s,a)$ is sampled  $n \in  O\left(\frac{|\mathcal{S}|}{\rho_t^2} + \frac{1}{\rho_t^2}\log\left(\frac{|\mathcal{S}||\mathcal{A}|}{\delta_t}\right)\right)$ times, then the following inequalities holds for any $(s, a)$: 
\begin{equation}
    \mathbf{P}\left(|{\overline{R}}(s,a) - \hat{R}_t(s,a)| \le \rho_t\right) \ge 1-\frac{\delta_t}{2|\mathcal{S}||\mathcal{A}|},
\end{equation}
\begin{equation}
\mathbf{P}\left(\left\Vert\mathbb{P}(s,a) - \hat{\mathbb{P}}_t(s,a)\right\Vert_1\le \rho_t\right) \ge 1 -\frac{\delta_t}{2|\mathcal{S}||\mathcal{A}|}.
\end{equation}
 Hence, using Lemma~\ref{lemma:value_bound} and applying union bounds, we obtain that the following holds  with at least probability $ 1 - (\delta^R_t  + \delta^P_t)|\mathcal{S}||\mathcal{A}| $
 \begin{align}
      V^*(s) - V^{\hat{\pi}^R_t}(s) &\le  2\rho_t\left( 2+ \frac{R_{\max}\gamma}{(1-\gamma)^2}  \right). 
 \end{align}
Setting $\rho_t = \frac{\epsilon_t}{2}{\left( 2+ \frac{R_{\max}\gamma}{(1-\gamma)^2}  \right)^{-1}}$ and  $\delta^R_t=\delta^P_t = \frac{\delta_t}{2|\mathcal{S}||\mathcal{A}|}$,  
 \begin{align}\label{eq:value_inequality}
       \mathbf{P}\left( V^*(s) - {V^{\hat{\pi}_t^R}(s)} \le \epsilon_t   \right) &\ge 1-\delta_t, \ \forall s \in \mathcal{S} \nonumber \\
    \implies  \mathbf{P}\left( \Vert V^*(s) - V^{\hat{\pi}_t^R}(s) \Vert_{\infty}\le \epsilon_t   \right) &\ge 1-\delta_t.
 \end{align}
Additionally, the order of $n$ in terms of $\epsilon_t$ becomes
$n \in O\left(\frac{|\mathcal{S}|}{\epsilon_t^2} + \frac{1}{\epsilon_t^2}\log\left(\frac{|\mathcal{S}||\mathcal{A}|}{\delta_t}\right)\right)$. $\hfill \blacksquare$

In Theorem~\ref{thm1}, we obtain the number of times $n$ each state-action pair needs to be visited to reduce the estimation uncertainty in rewards and transition probabilities to obtain an $\epsilon_t$-optimal policy with probability at least $1-\delta_t$. Now we estimate the total number of exploration steps that are needed to ensure that each state-action pair is visited at least $n$ times.


\begin{lemma}[\bit{Adapted from~\cite[Theorem 3]{chung2012chernoff}}]\label{lemma:markov_chain_bounds}
For an ergodic Markov chain with state space $\mathcal{S}$ and stationary distribution $\subscr{\phi}{ss}$, let $\tau=\tau(\sigma)$ be the $\sigma$-mixing time\footnote{$\sigma$-mixing time for an ergodic Markov chain in the minimal time until the distribution of Markov chain is $\sigma$-close in total variation distance to its steady state distribution~\cite{aldous1997mixing}.} with $\sigma\le\frac{1}{8}$. Let $\phi_0$ be the initial distribution on $\mathcal{S}$ and let $\Vert\phi_0 \Vert_{\subscr{\phi}{ss}} = \sqrt{\sum_{s \in \mathcal{S}}\frac{\phi_0(s)^2}{{\subscr{\phi}{ss}(s)}}} $. Let $\subscr{n}{vis}(s_i, N)$ 
be the number of times state $s_i \in \mathcal{S}$ is visited until time $N$. Then, for any $0\le \kappa \le 1$, there exists a constant $c>0$ (independent of $\sigma$ and $\kappa$) such that:

\begin{multline}\label{markov_bound}
    \mathbf{P}\left( \subscr{n}{vis}(s_i, N) \ge (1-\kappa)N\subscr{\phi}{ss}(s_i)  \right) \ge 1- c\times \\\Vert\phi_0 \Vert_{\subscr{\phi}{ss}}
    e^{-\frac{\kappa^2N\subscr{\phi}{ss}(s_i)}{72\tau}}.
\end{multline}
\end{lemma}

\begin{proof}
See~\cite[Theorem 3]{chung2012chernoff} for the proof.
\end{proof}




In the exploration epoch, at any state $s_i \in \mathcal{S}$, we choose  actions uniformly randomly. Consider the Markov chain on $\mc S$ that is associated with the uniform action selection policy. Let $\seqdef{\phi_0(s)}{s \in \mc S}$ and $\seqdef{\subscr{\phi}{ss}(s)}{s \in \mc S}$, respectively, be the associated initial and stationary distribution.  We can also consider an equivalent lifted Markov chain on $\mathcal{S}\times\mathcal{A}$ with states $(s_i,a_j)$  such that $s_i \in \mathcal{S}$ and $a_j \in \mathcal{A}$.
The lifted Markov chain has the initial and stationary distribution, $\phi_0(s,a) = \frac{\phi_0(s_i)}{|\mathcal{A}|}$ and $\subscr{\phi}{ss}(s,a) = \frac{\subscr{\phi}{ss}(s_i)}{|\mathcal{A}|}$, respectively. 
Hence, we can apply Lemma~\ref{lemma:markov_chain_bounds} to obtain the probability of visiting a state-action pair $(s,a)$ at least $(1-\kappa)N\subscr{\phi}{ss}(s,a)$ times after $N$ time steps {under the uniform action selection policy}.


We now design a sequence of exploration and exploitation epochs. Let $\alpha_i$ and $\beta_i$ be the lengths of the $i$-th exploration and exploitation epoch, respectively. 
{Let  $\subscr{\phi}{ss}^{\min} :=\min_{(s,a)\in \mathcal{S} \times \mathcal{A}}\subscr{\phi}{ss}(s,a)$ and $N_i = \frac{\bar{N}_i}{(1-\kappa) \subscr{\phi}{ss}^{\min}}$, where $\bar{N}_i$ is the upper bound in~\eqref{eq:n(s)} associated with $(\epsilon_i, \delta_i)$}. Let $\delta^{\alpha_i} := c\Vert\phi_0 \Vert_{\subscr{\phi}{ss}}e^{-\frac{\kappa^2 N_i\subscr{\phi}{ss}^{\min}}{72\tau}}$. Note that the desired values of $(1-\kappa)N\subscr{\phi}{ss}(s_i, a_j)$ and $\delta^{\alpha_i}$ can be obtained by tuning $N$ and $\kappa$ in \eqref{markov_bound}.

\begin{theorem}[\bit{Regret bound for DSEE algorithm}]\label{thm2}
Let the length of exploitation epochs in DSEE be exponentially increasing, i.e. $\beta_i=\eta^i, \ \eta>1$. 
Let $\epsilon_i=\eta^{-\frac{i}{3}}$ and $\delta_i = \eta^{-\frac{i}{3}}$ such that $\mathbf{P}(\Vert V^*(s)- {V^{\hat{\pi}_t^R}(s)} \Vert_{\infty} \le \epsilon_i) \ge 1-\delta_i$ after exploration epoch $i$. For any $\delta \in (0,1)$, set  $\delta^{\alpha_i} = \frac{6\delta}{|S||A|\pi^2i^2}$. 
Then, the cumulative regret for the DSEE algorithm $\mathbf{R}_T \in O((T)^{\frac{2}{3}} \log(T) )$ grows sub-linearly with time $T$ with probability at least  $1-\delta$.

\end{theorem}

\begin{proof}
We note that the system state at the start of the $i$-th exploration epoch might be different from the final state at the end of the $(i-1)$-th exploration epoch. Therefore, we remember the final state of the previous exploration epoch and wait for the same state to restart the new exploration epoch. For the ergodic MDP {under the uniform action selection policy} (assumption A2), we know that the expected hitting time is finite~\cite{chen2008expected}. Let $U \in \mathbb{R}_{>0}$ be a constant upper bound on the expected hitting time to reach the final state in the previous exploration epoch from an arbitrary initial state in the current exploration epoch.
Hence, the cumulative regret during the $i$-th exploration epoch of length $\alpha_i$ is upper-bounded by $(U+\alpha_i) {\mathfrak{R}}_{\max}$,  where ${\mathfrak{R}}_{\max}=\frac{R_{\max}}{1-\gamma}$ is the maximum instantaneous regret.

Since at start of the exploitation epoch $i$ of length $\beta_i$,  $\mathbf{P}(\Vert V^*(s) - {V^{\hat{\pi}_t^R}(s)}\Vert_{\infty} \le \epsilon_i) \ge 1-\delta_i$, the expected cumulative regret during the exploitation epoch is $(1-\delta_i)\beta_i \epsilon_i + \delta_i\beta_i{\mathfrak{R}}_{\max}$. Therefore, the total cumulative regret after $k$ sequences of exploration and exploitation each is upper bounded by:
\begin{align}\label{eq:regret}
    \mathbf{R}_{T_k} &\le \sum_{i=1}^{k}\left((\alpha_i+U){\mathfrak{R}}_{\max} + (1-\delta_i)\beta_i \epsilon_i + \delta_i\beta_i {\mathfrak{R}}_{\max}\right) \nonumber \\
    &\le \sum_{i=1}^{k}\left((\alpha_i+U){\mathfrak{R}}_{\max} + \beta_i \epsilon_i + \delta_i\beta_i {\mathfrak{R}}_{\max}\right).
\end{align}
Let $T_i$ be the time at the end of the $i$-th exploitation epoch. Then, 
$\sum_{j=1}^{k}\beta_j < T_k\le \sum_{j=1}^{k}(\alpha_j+U) + \sum_{j=1}^{k}\beta_j$. We design the length of the exploitation epochs to be exponentially increasing, i.e., $\beta_i=\eta^i$, for $\eta>1$. Thus, $T_k \in  O(\sum_{j=1}^{k}\eta^j ) = O(\eta^{k})$. 
Let $\epsilon_i=\eta^{-di}$ and $\delta_i = \eta^{-gi}$, where $d \in (0,1)$ and $g \in (0,1)$ are constants that we design later. 
Thus,~\eqref{eq:regret} can be written as:
\begin{align}\label{eq:regret_2}
    \mathbf{R}_{T_k} &\le {\mathfrak{R}}_{\max}\left(\sum_{i=1}^{k}\alpha_i+ kU\right) + \sum_{i=1}^{k}\eta^{i(1-d)} + \nonumber\\
    & \quad \quad \quad \quad \quad \quad \quad \quad \quad \quad {\mathfrak{R}}_{\max}\sum_{i=1}^{k}\eta^{i(1-g)}.
\end{align}
For a state-action pair $(s,a)$, where $s \in \mathcal{S}$ and $a \in \mathcal{A}$, let $\delta^{\alpha_i}_{(s,a)} := c\Vert\phi_0 \Vert_{\subscr{\phi}{ss}}e^{-\frac{\kappa^2 N_i\subscr{\phi}{ss}(s,a)}{72\tau}}$. Therefore, using Lemma~\ref{lemma:markov_chain_bounds}, at the end of the $i$-th epoch,  
 \begin{equation}\label{eq:markov_bound_delta}
 \mathbf{P}\left( \subscr{n}{vis}(s,a, N_i) \ge (1-\kappa)N_i\subscr{\phi}{ss}(s,a)  \right) \ge 1- \delta^{\alpha_i}_{(s,a)}.
 \end{equation}

 Recall $\delta^{\alpha_i} := c\Vert\phi_0 \Vert_{\subscr{\phi}{ss}}e^{-\frac{\kappa^2N_i\subscr{\phi}{ss}^{\min}}{72\tau}}$, where $\subscr{\phi}{ss}^{\min}:=\min_{(s,a)}\subscr{\phi}{ss}(s,a)$. Substituting $N_i = \frac{\bar{N}_i}{(1-\kappa) \subscr{\phi}{ss}^{\min}}$ in~\eqref{eq:markov_bound_delta}, 
  \begin{equation}\label{eq:markov_bound_continue}
 \mathbf{P}\left( \subscr{n}{vis}(s,a, N_i) \ge  \bar{N}_i \right) \ge 1- \sum_{m=1}^{|\mathcal{S}||\mathcal{A}|}\delta^{\alpha_i},
 \end{equation}
for each state-action pair $(s,a)$. {Therefore, in $N_i$ time steps of the lifted Markov chain, each $(s,a)$ is visited at least $\bar N_i$ times with high probability. Thus, $N_i$ is an upper bound on $\sum_{j=1}^i \alpha_j$ with probability in~\eqref{eq:markov_bound_continue}.  
Therefore, using union bounds, 
with high probability $1- \sum_{j=1}^{k}\sum_{m=1}^{|\mathcal{S}||\mathcal{A}|}\delta^{\alpha_j}$, $\sum_{j=1}^{k}\alpha_j \le N_k = \frac{\bar{N}_k}{(1-\kappa) \subscr{\phi}{ss}^{\min}}$, and hence,}
\begin{align}\label{eq:regret_3}
    \mathbf{R}_{T_k} &\le \frac{{\mathfrak{R}}_{\max} \bar N_k}{(1-\kappa)\subscr{\phi}{ss}^{\min}} + kU{\mathfrak{R}}_{\max}+\sum_{i=1}^{k}\eta^{i(1-d)} + \nonumber \\ & \quad \quad \quad \quad \quad  \quad \quad \quad \quad \quad \quad {\mathfrak{R}}_{\max}\sum_{i=1}^{k}\eta^{i(1-g)}.
\end{align}

Using Theorem~\ref{thm1}, $\bar N_k \in O\left(\frac{|\mathcal{S}|}{\epsilon_k^2} + \frac{1}{\epsilon_k^2}\log\left(\frac{|\mathcal{S}||\mathcal{A}|}{\delta_k}\right)\right)$.  
Therefore,

\begin{align}\label{eq:regret_4}
    \mathbf{R}_{T_k} &\le \frac{{\mathfrak{R}}_{\max}\lambda}{(1-\kappa)\subscr{\phi}{ss}^{\min}}\left(\frac{|\mathcal{S}|}{\epsilon_k^2} + \frac{1}{\epsilon_k^2}\log\left(\frac{|\mathcal{S}||\mathcal{A}|}{\delta_k}\right)\right) +  \nonumber \\
    &\ \ \ \ \ \ + kU{\mathfrak{R}}_{\max}+\sum_{i=1}^{k}\eta^{i(1-d)} + {\mathfrak{R}}_{\max}\sum_{i=1}^{k}\eta^{i(1-g)} \nonumber \\
    &\le \frac{{\mathfrak{R}}_{\max}\lambda}{(1-\kappa)\subscr{\phi}{ss}^{\min}}\left(\eta^{2dk}|\mathcal{S}| + \eta^{2dk}\log\left(\eta^{gk}|\mathcal{S}||\mathcal{A}|\right)\right) +  \nonumber \\
    &\ \ \ \ \ \ + kU{\mathfrak{R}}_{\max}+\sum_{i=1}^{k}\eta^{i(1-d)} + {\mathfrak{R}}_{\max}\sum_{i=1}^{k}\eta^{i(1-g)},
\end{align}
for some constant $\lambda$. Recall that $T_k \in O(\eta^{k})$, which implies $k \in O(\log(T_k))$. 
Let $Z$ be the right-hand side of~\eqref{eq:regret_4}.
Then, we have:
\begin{align}\label{eq:regret_5}
    Z &\in O\left( (T_k)^{2d} + (T_k)^{2d} \log(T_k) + (T_k)^{(1-d)} + (T_k)^{(1-g)} \right).   \nonumber \\
    &\in O((T_k)^{\frac{2}{3}} \log(T_k) ),
\end{align}
by choosing $d=g=\frac{1}{3}$. Hence, the cumulative regret  $\mathbf{R}_{T_k}\in O((T_k)^{\frac{2}{3}} \log(T_k) )$ grows sub-linearly with time $T_k$ with probability at least $1- \sum_{i=1}^{k}\sum_{m=1}^{|\mathcal{S}||\mathcal{A}|}\delta^{\alpha_i}$. 

Setting $\delta^{\alpha_i} = \frac{6\delta}{|S||A|\pi^2i^2}$, we have $\sum_{i=1}^{k}\sum_{m=1}^{|\mathcal{S}||\mathcal{A}|}\delta^{\alpha_i} \le \delta$.
\end{proof}

\section{Conclusions} \label{Conclusions}

We proposed a DSEE algorithm with interleaving exploration and exploitation epochs for model-based RL problems that aims to simultaneously learn the system model, i.e., an MDP, and the associated optimal policy. During exploration, we uniformly sample the action in each state and update the estimates of the mean rewards and transition probabilities. These  estimates are used in the exploitation epoch to obtain a robust policy with high probability. We designed the length of the exploration and exploitation epochs such that the cumulative regret grows as a sub-linear function of time. 

\bibliographystyle{IEEEtran}
\bibliography{mybib}
\end{document}